\documentclass[12pt]{article}
\usepackage{amsthm}
\usepackage{enumitem,amsmath,amsfonts,amssymb}
\usepackage{mathtools}
\usepackage{times}
\usepackage{graphicx}
\usepackage{color}
\usepackage{multirow}
\usepackage{hyperref}
\usepackage{diagbox}
\usepackage[authoryear]{natbib}
\usepackage{rotating}
\usepackage{bbm}
\usepackage{latexsym}
\usepackage{xcolor}
\newtheorem{theorem}{Theorem}[section]

\newtheorem{lemma}[theorem]{Lemma}

\theoremstyle{definition}
\newtheorem{definition}[theorem]{Definition}
\newtheorem{assumption}{Assumption}
\newtheorem{remark}[theorem]{Remark}

\newcommand{\bw}{\mathbf{w}}

\newcommand{\ibb}{\mathbb{I}}

\newcommand{\wcal}{\mathcal{W}}

\newcommand{\zcal}{\mathcal{Z}}

\newcommand{\dcal}{\mathcal{D}}
\newcommand{\bb}{c_2}
\newcommand{\aaa}{c_1}

\newcommand{\ebb}{\mathbb{E}}
\newcommand{\qbb}{\mathrm{Q}}
\newcommand{\pbb}{\mathrm{P}}

\newcommand{\be}{\mathbf{e}}
\newcommand{\bv}{\mathbf{v}}

\newcommand{\rbb}{\mathbb{R}}

\numberwithin{equation}{section}
\usepackage{graphicx}

\newcommand{\captionfonts}{\normalsize}
\newcommand{\zsj}{\color{orange}Sijia: }

\makeatletter
\long\def\@makecaption#1#2{%
  \vskip\abovecaptionskip
  \sbox\@tempboxa{{\captionfonts #1: #2}}%
  \ifdim \wd\@tempboxa >\hsize
    {\captionfonts #1: #2\par}
  \else
    \hbox to\hsize{\hfil\box\@tempboxa\hfil}%
  \fi
  \vskip\belowcaptionskip}
\makeatother

\begin{document}
\hspace{13.9cm}1

\ \vspace{20mm}\\

{\LARGE Randomized Pairwise Learning with Adaptive Sampling: A PAC-Bayes Analysis}

\ \\
{\bf \large Sijia Zhou$^{\displaystyle 1}$\quad Yunwen Lei$^{\displaystyle 2}$ \quad  Ata Kab\' an$^{1}$}\\
{$^{\displaystyle 1}$School of Computer Science, University of Birmingham, United Kingdom.}\\
{$^{\displaystyle 2}$Department of Mathematics, The University of Hong Kong, China.}\\
%


\thispagestyle{empty}
\markboth{}{NC instructions}
\ \vspace{-0mm}\\
%
\begin{center} {\bf Abstract} \end{center}
We study stochastic optimization  with data-adaptive sampling schemes to train pairwise learning models. Pairwise learning is ubiquitous, and it covers several popular learning tasks such as ranking, metric learning and AUC maximization. A notable difference of pairwise learning from pointwise learning is the statistical dependencies among input pairs, for which existing analyses have not been able to handle in the general setting considered in this paper.
To this end, we extend recent results that blend together two algorithm-dependent frameworks of analysis -- algorithmic stability and PAC-Bayes -- which allow us to deal with any data-adaptive sampling scheme in the optimizer.
We instantiate this framework to analyze
(1) pairwise stochastic gradient descent, which is a default workhorse in many machine learning problems, and (2) pairwise stochastic gradient descent ascent, which is a method used in adversarial training. All of these algorithms make use of a stochastic sampling from a discrete distribution (sample indices) before each update. Non-uniform sampling of these indices has been already suggested in the recent literature, to which our work provides generalization guarantees in both smooth and non-smooth convex problems.

\section{Introduction 
}


The increasing availability of data makes it feasible to use increasingly large models in principle. However, this comes at the expense of an increasing computational cost of training these models in large pairwise learning applications. Some notable examples of pairwise learning problems include ranking, AUC maximization, and metric learning \citep{agarwal2009generalization,clemenccon2008ranking,cortes2004auc,cao2012generalization}. For instance, in metric learning we aim to learn an appropriate distance or similarity to compare pairs of examples, which has numerous applications such as face verification and person re-identification (Re-ID) \citep{koestinger2012large, xiong2014person, guillaumin2009you, zheng2017discriminatively, yi2014deep}.

Stochastic gradient descent (SGD) and Stochastic Gradient Descent Ascent (SGDA) are often the methods of choice in large-scale minimization and min-max optimization problems in machine learning for their favourable time efficiency. These algorithms use sampling strategies to build estimates of the true gradient to improve efficiency. Some recent works propose data-dependent sampling strategies to speed up convergence and improve the accuracy of models in the optimization context \citep{zhao2015stochastic, allen2016even,katharopoulos2017biased,johnson2018training,wu2017sampling,han2022rethinking}.  


In pairwise learning, the empirical risk takes the form of a second-order $U$-statistic. Therefore, results on $U$-processes can be used to investigate the generalization analysis of pairwise learning  \citep{clemenccon2008ranking,de2012decoupling}. While there is much research on the generalization analysis of pairwise learning, the effect of non-uniform, data-dependent sampling schemes remains unclear and has not been rigorously studied.

The adaptive choice of the sampling distribution can be important in noisy data situations where the training points are not equally reliable or informative. In rare cases when the value of individual points is known, then the sampling distribution can be designed and fixed before training. In most realistic cases, however, it is desirable to learn the sampling distribution together with training the model. The idea of sampling shows great potential in the literature for randomized algorithms such as 
SGD \citep{zhao2015stochastic} and  SGDA \citep{beznosikov2023stochastic}. SGDA is one of the most popular methods for minimax problems, which has many applications such as generative adversarial networks (GANs) \citep{goodfellow2014generative} and adversarial training \citep{madry2017towards}.
Adversarial perturbations are subtle, often imperceptible modifications to input data designed to deceive models and cause incorrect predictions \citep{goodfellow2014explaining}. Recent studies in pairwise learning have explored strategies to enhance adversarial robustness, applying adversarial pairwise learning methods to minimax problems across various domains, such as metric learning \citep{zhou2022enhancing,WEN2025106955}, ranking \citep{liu2019geo,WEN2025106955}, and kinship verification \citep{zhang2020advkin}. These developments underscore the need for further investigation into improving the robustness of pairwise models against adversarial attacks.

Importance sampling is one of the widely used strategies of working with a distribution proportional to the gradient norm to minimize the variance of the stochastic gradients to achieve faster convergence rate \citep{zhao2015stochastic,katharopoulos2017biased}.  Therefore,  recent work \citep{zhou2023toward,london2017pac} begun to develop a better understanding of the generalization behavior of such algorithms, and in this work we extend the novel analytic tools that enable such analysis to the context of pairwise learning, with general sampling schemes. The main bottleneck in the analysis of adaptive non-uniform sampling based stochastic optimizers is the requirement of a correction factor to ensure the unbiasedness of the gradient, as this factor depends on training data points. In addition, in the pairwise setting we also need to cater to statistical dependencies between data pairs, which is due to the fact that each point participates in multiple pairs.

To tackle these problems, we develop a PAC-Bayesian analysis of the generalization of stochastic optimization methods, which removes the need for a correction factor, 
and we use $U$-statistics to capture the statistical structure of pairwise loss functions.
The PAC-Bayes framework allows us to obtain generalization bounds that hold uniformly for all posterior sampling schemes, under a mild condition required on a pre-specified prior sampling scheme. For randomized methods, such as SGD and SGDA, the 
sampling indices will be considered as the hyperparameters that follow a sampling distribution. The uniform sampling distribution makes natural prior, and the PAC-Bayes framework allows us to leverage this prior to obtain bounds that hold for arbitrary data-dependent posterior sampling, provided a mild condition on the prior sampling.

However, the previous work in the above framework only considered the classic pointwise learning setting, which cannot tackle the dependencies in the objective function of pairwise learning.
In this paper we enable this 
using a moment bound for uniformly stable pairwise learning algorithms \citep{lei2020sharper}, which is based on a new decomposition of the objective function using the properties of second-order $U$-statistics. Blending this into the PAC-Bayes methodology will give us bounds that hold for general randomized predictors over inputs.

Our results on pairwise SGD and SGDA follow the above framework, upon verifying a sub-exponential stability condition w.r.t. a prior sampling in these algorithms. Our main results are listed in Table \ref{table:rate}, summarizing the generalization bounds of the order $\widetilde{O}(1/\sqrt{n})$ for these randomized algorithms under different assumptions, where $n$ is the sample size. 

Our technical contributions are summarized as follows:
\begin{itemize}
    \item We bound the generalization gap of randomized pairwise learning algorithms that operate with an arbitrary data-dependent sampling, in a PAC-Bayesian framework, under a sub-exponential stability condition.
    \item We apply the above general result to pairwise SGD and pairwise SGDA with arbitrary sampling. For both of these algorithms, we verify the sub-exponential stability in both smooth and non-smooth problems.
\end{itemize}

The remainder of the paper is organized as follows. We survey the related work on the generalization analysis and non-uniform sampling in Section \ref{sec:relted}. We give a brief background on $U$-statistics and algorithmic stability analysis in Section~\ref{sec:pre}. Our general result and its applications to SGD and SGDA are presented in Section \ref{sec:main}.

\begin{table}[ht]
\renewcommand{\arraystretch}{0.65}
\begin{center}
\begin{tabular}{|l|r|ll|r|}\hline
Algo. &   Asm.&  \multicolumn{2}{|c|}{ Time $T$ and step size $\eta$ } & Rates
\\ \hline
\multirow{ 4}{*}{SGD}    &  \multirow{ 2}{*}{L,C}    &   \multirow{ 2}{*}{$T =\Theta( n^2)$}   &    \multirow{ 2}{*}{$\eta =\Theta(T^{-\frac{3}{4}})$}    &  $\widetilde{O}(1/\sqrt{n})$
   \\   &  & &  &   Thm. \ref{cor_sgd_smooth} 1) \\  \cline{2-5}  &  \multirow{ 2}{*}{L,S,C}   &    \multirow{ 2}{*}{$T=\Theta(n)$}   &   \multirow{ 2}{*}{$\eta=\Theta(T^{-\frac{1}{2}})$}      &  $\widetilde{O}(1/\sqrt{n})$    \\   &  & &  &  Thm. \ref{cor_sgd_smooth} 2) \\  \hline \multirow{4}{*}{SGDA}     &   \multirow{ 2}{*}{L,C}     &   \multirow{ 2}{*}{$T=O\left(n^{2}\right)$}  &  \multirow{ 2}{*}{$\eta=O(T^{-\frac{3}{4}} )$ } &      $\widetilde{O}(1/\sqrt{n})$ \\    & & &  &  Thm. \ref{cor_sgda_smooth} 1) \\  \cline{2-5}  &   \multirow{ 2}{*}{L,S,C}   &     \multirow{ 2}{*}{$T=O(n)$}    &   \multirow{ 2}{*}{$\eta=O(T^{-\frac{1}{2}})$}  &   $\widetilde{O}(1/\sqrt{n})$   \\    & & &  &   Thm. \ref{cor_sgda_smooth} 2) \\   \hline
\end{tabular}
\end{center}
\caption{Summary of generalization rates obtained for two pairwise stochastic optimization algorithms (SGD, SGDA) under two sets of assumptions (Lipschitz (L), smooth (S), convex (C)) on the pairwise loss function, together with the chosen number of iterations $T$ and step size $\eta$. The sample size is $n$, and $d$ represents the dimension of the parameter space. According to this summary, we notice that smaller step sizes and more iterations are needed if smoothness assumption is removed (more details in Section \ref{sec:main}).  }
\label{table:rate}
\end{table}


\section{Related Work \label{sec:relted} } 
\textbf{Non-uniform Sampling in Randomized Algorithms}.
Importance sampling \citep{zhao2015stochastic} is a popular sampling strategy in stochastic optimization, where samples are chosen with the likelihood proportional to the norms of their gradients. 
This method aims to reduce the variance of the stochastic gradient and accelerate the training process. However, this method can be computationally prohibitive in practice. There are works that propose the approximations of the true gradients, adapting the importance sampling idea to reduce the computational cost with different strategies~\citep{johnson2018training, katharopoulos2018not}.

Furthermore, some works employ the loss as an important metric for sampling distribution to achieve faster convergence \citep{zhao2015stochastic,katharopoulos2017biased,london2017pac}. Some propose a novel upper bound of the gradient norm, with better performance than sampling data proportional to the loss \citep{katharopoulos2018not}. The work in \cite{wu2017sampling} proposes to choose the samples uniformly based on their relative distance to each other. 
Among non-uniform sampling, data-dependent sampling is attracting growing interests due to its practical potential. Moreover, such sampling schemes can also be applied to coordinate selection optimization methods \citep{salehi2018coordinate, allen2016even}. A new sampling approach called group sampling was introduced for unsupervised person Re-ID  to solve the negative effect that lies in random sampling~\citep{han2022rethinking}. However, there are few results on the generalization analysis for the resulting randomized algorithms, which is our goal in this paper.

\textbf{Generalization through Algorithmic Stability}.
Stability was popularized in the seminal work of \cite{bousquet2002stability}, to formalize the intuition that, algorithms whose output is resilient to changing an example in its input data will generalize. The stability framework subsequently motivated a chain of analyses of randomized iterative algorithms, such as SGD~\citep{hardt2016train} 
and SGDA~\citep{lei2021stability,farnia2021train}. While the stability framework is well suited for SGD-type algorithms that operate a uniform sampling scheme, this framework alone is unable to tackle data-dependent arbitrary sampling schemes.

\textbf{Generalization through PAC-Bayes}. The PAC-Bayes theory of generalization is another algorithm-dependent framework in statistical learning, the gist of which is to leverage a pre-specified prior distribution on the parameters of interest to obtain generalization bounds that hold uniformly for all posterior distributions \citep{shawe1997pac,mcallester1999some}. Its complementarity with the algorithmic stability framework sparked ideas for combining them
\citep{london2016stability,rivasplata2018pac, sun2022stability,oneto2020randomized}, some of which are also applicable to randomized learning algorithms such as SGD and SGLD~\citep{mou2018generalization,london2017pac,negrea2019information,li2019generalization}. While insightful, these works assume i.i.d. examples, and cannot be applied to non-i.i.d. settings such as pairwise learning.

In non-i.i.d. settings, \cite{ralaivola2010chromatic} gave PAC-Bayes bounds using fractional covers, which allows for handling the dependencies within the inputs. This gives rise to generalization bounds for pairwise learning, with predictors following a distribution induced by a prior distribution on the model's parameters. However, with SGD-type methods in mind, which have a randomization already built into the algorithm, the classic PAC-Bayes approach of placing a prior on a model's parameters would be somewhat artificial. Instead the construction proposed in \cite{london2017pac, zhou2023toward} (in i.i.d. setting) is to exploit this built-in stochasticity directly, by interpreting it as a PAC-Bayes prior placed on a hyperparameter. We will build on this idea further in this work.

\section{Preliminaries \label{sec:pre}}

\subsection{Pairwise Learning and $U$-statistics}
Let $\mathcal{D}$ be an unknown distribution on sample space $\mathcal{Z}$.  We denote by
$\wcal \subseteq \rbb^d$ the parameter space, and $\Phi$ will be a hyperparameter space. 
Given a training set $S = \{z_1,\ldots, z_n\}$ drawn i.i.d. from $\mathcal{D}$, and a hyperparameter $\phi\in\Phi$, a learning algorithm $A$ returns a model parameterised by $A(S;\phi )\in\wcal$. 

We are interested in pairwise learning problems, and will use a pairwise loss function $\ell:\wcal\times\zcal\times\zcal\mapsto\rbb_+$ to measures the mismatch between the prediction of model that acts on example pairs. 
The generalization error, or risk, is defined as the expected loss of the learned predictor applied on an unseen pair of inputs drawn from $\dcal^2$, that is
\begin{equation}
R(A(S;\phi)) := \mathbb{E}_{z,\tilde{z} \sim \mathcal{D}} [\ell(A(S;\phi),z,\tilde{z})].\label{Risk}
\end{equation}
Since $\mathcal{D}$ is unknown, we consider the empirical risk,
\begin{equation}
R_S(A(S;\phi)) := \frac{1}{n(n-1)}\sum_{i,j\in[n]:i\neq j}\ell(A(S;\phi),z_i,z_j),\label{Remp}
\end{equation}
where $[n]:=\{1,\ldots,n\}$.
The generalization error is a random quantity as a function of the sample $S$, which doesn't consider the randomization used when selecting the data or feature index for the update rule of $A$ at each iteration.

To take advantage of the built-in stochasticity of the type of algorithms we consider,
we further define two distributions on the hyperparameter space $\Phi$: a sample-independent distribution $\pbb$, and a sample-dependent distribution $\qbb$. 
In this stochastic or randomized learning algorithm setting, the expected risk, and the expected empirical risk (both w.r.t $\qbb$) are defined  as
\begin{equation*}
\emph{R}( \qbb) = \mathop{\mathbb{E}}\limits_{\phi \sim \qbb} [R(A(S;\phi))], \quad  \emph{R}_S( \qbb) = \mathop{\mathbb{E}}\limits_{\phi \sim \qbb} [R_S(A(S;\phi))].\end{equation*}
We denote the difference between the risk and the empirical risk (i.e. the generalization gap) by $G(S,\phi) := R(A(S;\phi)) - R_S(A(S;\phi))$. 

The difficulty with the pairwise empirical loss \eqref{Remp} is that, even with $S$ consisting of i.i.d. instances, the pairs from $S$ are dependent of each other. Instead, $R_S(A(S;\phi)) $ is a second-order $U$-statistic. A powerful technique to handle $U$-statistic is the representation as an average of “sums-of-i.i.d.” blocks~\citep{de2012decoupling}. That is, for a symmetric kernel $q:\zcal\times\zcal\mapsto\rbb$, we can represent the $U$-statistic $U_n:=\frac{1}{n(n-1)}\sum_{i,j\in[n]:i\neq j}q(z_i,z_j)$ as
\begin{align}
\label{eq:u-sta_decomp}
U_n   = \frac{1}{n!} \sum_{\sigma} \frac{1}{\lfloor n/2 \rfloor} \sum_{i=1}^{\lfloor n/2 \rfloor}q(z_{\sigma (i)},z_{\sigma(\lfloor\frac{n}{2}\rfloor +i)}),
\end{align}
where $\sigma$ ranges over all permutations of $\{1,\ldots,n\}$.

\subsection{Connection with the PAC-Bayesian Framework}

As described above, we consider two probability distributions on the hyperparameters space $\Phi$, to account for the stochasticity in stochastic optimization algorithms, such as SGD and SGDA, where the hyperparameters $\phi \in \Phi$ are discrete distributions on indices. For instance in SGD, in every iteration $t \in [T]$, we have $\phi_t =(i_t, j_t)$ that is a pair of independently sampled sample indices, drawn from $\{(i_t, j_t) : i_t, j_t \in [n], i_t \neq j_t\}$ with replacement (more details in Section \ref{app_sgd}). The two distributions we defined on $\Phi$, namely $\pbb$ that needs to be specified before seeing the training data, distribution $\qbb$, which is allowed to depend on the samples, will be our PAC-Bayes prior, and PAC-Bayes posterior distributions respectively.
This setting is different from the classic use of PAC-Bayes in that the two distributions are  directly placed on the trainable parameter space $\wcal$. Our distributions defined on $\Phi$ indirectly induce distributions on the parameter estimates, without the need to know their parametric form. This setting of PAC-Bayes was formerly introduced in \citet{london2017pac} in a combination with algorithmic stability and further improved in our previous work \citep{zhou2023toward} treating part of the algorithms we consider here in the pointwise case.

\subsection{Connection with the Algorithmic Stability Framework}

A more recent framework for generalization problem considers algorithmic stability \citep{bousquet2002stability}, which measures the sensitivity of a learning algorithm to small changes in the training data. The concept considered in our work among several notions of algorithmic stability is uniform stability.
\begin{definition}[Uniform Stability\label{defn_uniform}]
For $\forall \phi$, we say an algorithm $A:S\mapsto A(S;\phi)$ is $\beta_{\phi}$-uniformly stable if
\begin{equation}\label{eq_defn_uniform}
|\ell(A(S;\phi),z,\tilde{z})-\ell(A(S',\phi),z,\tilde{z})|\leq \beta_{\phi},\quad \forall z,\tilde{z}\in \mathcal{Z},
\end{equation}
where $S,S'\in \mathcal{Z}^n$ differs by at most a single example.
\end{definition}

The algorithmic stability framework is suitable for analysing certain deterministic learning algorithms, or randomized algorithms with a pre-defined randomization.
In turn, here we are concerned with inherently stochastic algorithms where we wish to allow any data-dependent stochasticity, such as the variants of importance sampling and other recent practical methods mentioned in the related works \citep[e.g.][]{zhao2015stochastic,katharopoulos2018not,wu2017sampling,allen2016even,han2022rethinking}. 
Moreover, in principle our framework and results are applicable even if the sampling distribution is learned from the training data itself. 

\textbf{Sub-exponential Stability}.
A useful definition of stability that captures the stochastic nature of the algorithms we are interested in is the sub-exponential stability introduced in~\citet{zhou2023toward}.
Recall that $\phi$ is a random variable following a distribution defined on $\Phi$. Therefore, the stability parameter $\beta_\phi$ is also a random variable as a function of $\phi$. We want to control the tail behaviour of $\beta_{\phi}$ around a value that decays with the sample size $n$, and define the sub-exponential stability as the following. 
\begin{assumption}[{Sub-exponential stability}]\label{ass:beta-theta}
Fix any prior distribution $\pbb$ on $\Phi=\prod_{t=1}^{T}\Phi_t$. We say that a stochastic algorithm is sub-exponentially $\beta_{\phi}$-stable (w.r.t. $\pbb$) if, given any fixed instance of $\phi \sim \pbb$, it is $\beta_{\phi}$-uniformly stable, and there exist $\aaa, \bb\in\rbb$ such that for any  $\delta\in(0,1/n]$, the following holds with probability at least $1-\delta$
  \begin{equation}\label{beta-theta}
  \beta_\phi\leq \aaa +\bb\log(1/\delta).
  \end{equation}
\end{assumption}

\section{Main Results \label{sec:main}} 

In this section, we will give generalization bounds for SGD and SGDA in pairwise learning. To this aim, we first give a general result (Lemma~\ref{thm:main}) to show the connection between the sub-exponential stability assumption (Assumption \ref{ass:beta-theta}) and the generalization gap for pairwise learning. We then derive stability bounds to show that this assumption holds for SGD and SGDA, in both smooth convex and non-smooth convex cases. Based on these, we apply the stability bounds to Lemma \ref{thm:main} to derive the corresponding generalization bounds.  We use $K\lesssim K'$ if there exists a universal constant $a>0$ such that $K\leq a K'$. The proof is given in Appendix~\ref{app:main}.

\begin{lemma}[{Generalization of randomized pairwise learning}\label{thm:main}]
Given distribution $\pbb$, $\aaa, \bb>0$, and $M$-bounded loss for a sub-exponentially stable algorithm $A$, $\forall \delta\in(0,1/n)$, with probability at least $1-\delta$, the following holds uniformly for all $\qbb$ absolutely continuous w.r.t. $\pbb$,
\begin{equation*}\ebb_{\phi\sim\qbb}\left[G(S,\phi)\right]
\lesssim \left(\!\rm{KL}(\qbb\|\pbb)+\log\frac{1}{\delta}\right)\max\left\{\aaa\log n+\bb \log^2n,\frac{M}{\sqrt{n}}\right\},\end{equation*}
  where $\rm{KL}(\qbb\|\pbb)$ is the KL divergence between $\pbb$ and $\qbb$
\[
\rm{KL}(\qbb\|\pbb):=\int_{\phi \in \Phi}\log \frac{d\qbb}{d\pbb} \, d\qbb.
\]
\end{lemma}
In our case both $\pbb$ and $\qbb$ are discrete distributions on $\Phi$, so
$
\rm{KL}(\qbb\| \pbb):=\sum_{\phi \in \Phi} \qbb \log  \frac{\qbb}{\pbb}$, and $\pbb$ will be the uniform distribution on the finite set $\Phi$, so the absolute continuity condition is always satisfied, and also $\rm{KL}(\qbb\| \pbb) < \infty$.



A strength of Lemma~\ref{thm:main} is that we only need to check the sub-exponential stability under a prior distribution $\pbb$, which is often chosen to be the uniform distribution. Then Lemma~\ref{thm:main} automatically transfers it to generalization bounds for learning with any posterior distribution $\qbb$. To apply Lemma~\ref{thm:main} for a learning algorithm, it suffices to estimate the corresponding algorithmic stability, and verify that it satisfies the sub-exponential stability.  Before giving stability bounds, we introduce some assumptions. Let $\|\cdot\|_2$ denote the Euclidean norm. Let $S$ and $S'$ be neighboring datasets (i.e. differ in only one example, which we denote as the $k$-th example, $k\in[n]$). 
\begin{assumption}[Lipschitz continuity\label{defn_lipschitz}] Let $L>0$. We say $\ell$ is $L$-Lipschitz if for any $\bw_1$, $\bw_2$ $\in \wcal$, we have $
|\ell(\bw_1 )-\ell(\bw_2  )|\leq L\|\bw_1 - \bw_2\|_2.$ \end{assumption}
\begin{assumption}[Smoothness\label{defn_smooth}]
Let $\alpha\geq0$. We say a differentiable function $\ell$ is $\alpha$-smooth, if for any $\bw_1$, $\bw_2$ $\in \wcal$, $\|\nabla \ell(\bw_1   ) - \nabla \ell(\bw_2 )\|_2\leq \alpha\| \bw_1 - \bw_2\|_2,$
where $\nabla \ell$ represents the gradient of $\ell$.
\end{assumption}
\begin{assumption}[Convexity\label{defn_strongly_convex}]
We say $\ell$ is convex if the following holds $\forall \bw_1, \bw_2\in \wcal$,
\begin{equation*}
\ell(\bw_1 ) \geq \ell(\bw_2 ) +  \big\langle \nabla \ell(\bw_2 ), \bw_1 - \bw_2 \big\rangle,
\end{equation*}
where $\langle \cdot,  \cdot \rangle$ represents the inner product.
\end{assumption}

\subsection{Stability and Generalization of SGD \label{app_sgd}}
We now consider pairwise SGD, which, as we will show, also satisfies the sub-exponential stability in both smooth and non-smooth cases.

We denote $\bw_1$ an initial point and a uniform distribution over $\left([n]\times[n]\right)^T$.
At the $t$-th iteration for SGD, a pair of sample indices $\phi_t =(i_t, j_t)$ is uniformly randomly selected from the set $\{(i_t, j_t) : i_t, j_t \in [n], i_t \neq j_t\}$. This forms a sequence of index pairs $\phi=  (\phi_1,...,\phi_T)$. For step-size $\eta_t$, the model is updated  by $\bw_{t+1} = \bw_t - \eta_t \nabla \ell(\bw_t;z_{i_t},z_{j_t}).$

The following lemma shows that SGD with uniform sampling applied to smooth and non-smooth problems enjoy the sub-exponential stability. The proof is given in Appendix \ref{prf:sgd}.

\begin{lemma}[{Sub-exponential stability of pairwise SGD}]\label{sta_sgd}
Let $\{\bw_t\}, \{\bw_t'\}$ be two  sequences produced by SGD with uniform distribution $\pbb$ on neighboring $S$ and $S'$, respectively. Let Assumption \ref{defn_lipschitz} and Assumption~\ref{defn_strongly_convex} hold.
\begin{itemize}
    \item[1)]At the $t$-th iteration, with fixed step sizes, 
Assumption \ref{ass:beta-theta} holds with 
\[ \aaa\!=\!2\sqrt{e}L^2\eta (\sqrt{t}+ 2t/n) \quad \text{and} \quad \bb\! =\! 4\sqrt{e}L^2\eta\big(1+2(t/n)^{\frac{1}{2}}\big) .\]
 \item[2)] In addition, if  the Assumption \ref{defn_smooth} holds and $\eta \leq 2/\alpha$, at t-th iteration, 
Assumption~\ref{ass:beta-theta} holds with 
 \[c_1= 4L^2\eta t/n\quad \text{and} \quad\bb=4L^2\eta\big(1+2(t/n)^{\frac{1}{2}}\big).\]
\end{itemize}
\end{lemma}

Based on the above lemma and Lemma~\ref{thm:main}, we obtain the following generalization bound for pairwise SGD with general sampling.

\begin{theorem}[{Generalization bounds for pairwise SGD}\label{cor_sgd_smooth}]

Assume $\ell$ is $M$-bounded and Assumptions \ref{defn_lipschitz}, \ref{defn_strongly_convex} hold w.r.t $\ell$.
For any $\delta \in (0, 1)$ and uniform prior  distribution $\pbb$, with probability at least $1-\delta$ over $S$, $S \sim \mathcal{D}^n$, the following bounds hold for SGD with fixed step sizes and all posterior sampling distribution $\qbb$ on $\left([n]\times[n]\right)^T$.
\begin{itemize}
    \item[1)]
We have
\[  \ebb_{ \qbb}\left[G(S,\phi)\right]\!\lesssim \!\left( \!\rm{KL}(\qbb\|\pbb)\!+\!\log\frac{1}{\delta}\!\right) \!\max\!\left\{\!L\eta \Big(\sqrt{T}\!+\!\frac{T}{n}\!+ \!\sqrt{\frac{T}{n}
}\Big)\log^2 n,\frac{M}{\sqrt{n}}\!\right\}.
\]
\item[2)]In addition, if  the Assumption \ref{defn_smooth} holds, $\eta \leq 2/\hat{\alpha}$, we have
\[\ebb_{\qbb}\left[G(S,\phi)\right]\lesssim \left(\!\rm{KL}(\qbb\|\pbb)\!+\!\log\frac{1}{\delta}\!\right) \max\left\{L\eta \Big(\frac{T}{n} \!+ \!1\!+\!\sqrt{\frac{T}{n}}\Big)  \log^2 n, \frac{M}{\sqrt{n}}\right\}.
\]\end{itemize}
\end{theorem}

\begin{remark}
According to the choice of parameters suggested by \cite{lei2021generalizationb}, if we choose $\eta=\Theta(T^{-\frac{3}{4}})$ and $T=\Theta( n^2)$ in the non-smooth case (part 1), then the above theorem implies bounds of the order $\widetilde{O}(1/\sqrt{n})$.
In the smooth case (part 2), according to an analysis of the trade-off between optimization and generalization, \citet{lei2020sharper} suggested setting $T=\Theta(n)$ and $\eta=\Theta(1/\sqrt{T})$ to get an SGD iterate with a good generalization performance. With these choices, our bounds in Theorem \ref{cor_sgd_smooth}   are of order $\widetilde{O}(1/\sqrt{n})$, which are not improvable in general.  
\end{remark}

\subsection{Stability and Generalization of SGDA \label{app_sgda}}

In this subsection, we discuss SGDA for solving minimax problems in the convex-concave case. We will abuse the notations to apply them to the minimax case.
We receive a model $A(S;\phi):=\left(A_{\mathrm{w}}(S;\phi), A_{\mathrm{v}}(S;\phi)\right) \in \mathcal{W} \times \mathcal{V}$ by applying a learning algorithm $A$ on training set $S$ and measure the performance w.r.t. loss  $\ell: (\bw, \bv) \mapsto  \ell(\bw, \bv ; z,\tilde{z})$.
For any $\phi\in\Phi$, we consider the risk defined as
\[ \min _{\mathbf{w} \in \mathcal{W}} \max _{\mathbf{v} \in \mathcal{V}} R(A_{\mathbf{w}}(S;\phi), A_{\mathbf{v}}(S;\phi)):=\mathbb{E}_{z,\tilde{z} \sim \mathbb{D}}[\ell\left(A_{\mathbf{w}}(S;\phi), A_{\mathbf{v}}(S;\phi); z,\tilde{z} \right)].\]
We consider the following empirical risk as the approximation:
\[  R_{S}(A_{\mathbf{w}}(S;\phi), A_{\mathbf{v}}(S;\phi)) :=\frac{1}{n(n-1)}\sum_{i,j\in[n]:i\neq j} \ell\left(A_{\mathbf{w}}(S;\phi), A_{\mathbf{v}}(S;\phi); z_i,z_j\right).\]
We consider SGDA with a general sampling scheme, where the random index pairs follow from a general distribution



We denote $\bw_1$ and $\bv_1$ the initial points. Let $\nabla_{\bw}\ell$ and $\nabla_{\bv}\ell$ be the gradients w.r.t. $\bw$ and $\bv$ respectively.
Let $\pbb$ be a uniform distribution over $\left([n]\times[n]\right)^T$ and $S$ be a training dataset with $n$ samples. Let $(i_t, j_t)$ from set $\{(i_t, j_t) : i_t, j_t \in [n], i_t \neq j_t\}$ be drawn uniformly at random. At the $t$-th iteration, with step-size sequence $\{\eta_t\}$, SGDA updates the model as follows
\[
\begin{cases} \bw_{t+1}= \bw_{t}-\eta_{t}\nabla_{\bw}\ell(\bw_{t},\bv_{t};z_{i_t},z_{j_t})
,\\\bv_{t+1}= \bv_{t}+\eta_t\nabla_{\bv}\ell(\bw_{t},\bv_{t};z_{i_t},z_{j_t}).\end{cases}
\]

Before giving the results for SGDA, we introduce some assumptions w.r.t. both $\bw$ and $\bv$  \citep{farnia2021train, zhang2021generalization}.
\begin{assumption}
[Lipschitz continuity\label{def:sgda:lip}]
Let $L \geq0$. We say a differentiable function $\ell$ is $L$-Lipschitz, if for any $z, \tilde{z} \in \mathcal{Z}$, $\bw \in \mathcal{W}$, $\bv \in \mathcal{V}$ we have
\[\left\|\nabla_{\bw} \ell(\bw, \bv ; z,\tilde{z})\right\|_{2} \leq L \quad \text{and} \quad \left\|\nabla_{\bv} \ell(\bw, \bv ; z,\tilde{z})\right\|_{2} \leq L.\]
\end{assumption}

\begin{assumption}[Smoothness\label{sgda:smooth}]
Let  $ \alpha>0$. We say a differentiable function $ \ell$ is $\alpha$-smooth if the following inequality holds for any  $ \bw_1$, $\bw_2 \in \mathcal{W}$, $\bv_1$, $\bv_2 \in \mathcal{V}$ and  $z,\tilde{z} \in \mathcal{Z}$
\[\left\|\left(\arraycolsep=1.4pt\def\arraystretch{0.5}\begin{array}{c}
\nabla_{\bw } f(\bw_1, \bv_1; z,\tilde{z})-\nabla_{\bw } f\left(\bw_2, \bv_2; z,\tilde{z}\right) \\
\nabla_{\bv } f(\bw_1, \bv_1; z,\tilde{z})-\nabla_{\mathbf{v}} f\left(\bw_2, \bv_2; z,\tilde{z}\right)
\end{array}\right)\right\|_{2} \leq \alpha \left\|\left(\arraycolsep=1.4pt\def\arraystretch{0.5}\begin{array}{c}
\bw_1-\bw_2 \\
\bv_1-\bv_2
\end{array}\right)\right\|_{2}.\]

\end{assumption}

\begin{assumption}[Convexity-Concavity\label{defn_strongly_convex_sgda}]We say $\ell$ is concave if $-\ell$ is convex. We say $\ell$  is   convex-concave if $\ell(\cdot, \mathbf{v})$ is convex for every $\mathbf{v} \in \mathcal{V}$ and $\ell( \mathbf{w},\cdot)$ is concave for every $\mathbf{w} \in \mathcal{W}$.
\end{assumption}

Now we apply Lemma \ref{thm:main} to develop bounds for SGDA in both smooth and non-smooth cases. In the following lemma to be proved in Appendix \ref{prf:sgda}, we give stability bounds for SGDA and show these bounds satisfy Assumption \ref{ass:beta-theta}.

\begin{lemma}[{Sub-exponential stability of pairwise SGDA}\label{lem:sta_sgda}]
Let $\{\bw_t,\bv_t\}, \{\bw_t',\bv_t'\}$ be the sequences produced by SGDA on $S$ and $S'$ respectively with uniform distribution $\pbb$ and fixed step sizes. Let Assumption \ref{def:sgda:lip} and Assumption \ref{defn_strongly_convex_sgda} hold.
\begin{itemize}
    \item[1)]
At the $t$-th iteration, 
Assumption \ref{ass:beta-theta} holds with \[\aaa = 2\sqrt{2e}L^2\eta (\sqrt{t}+ 2t/n) \quad \text{and} \quad \bb= 4\sqrt{2e}L^2\eta (1+\sqrt{ 2t/n}).\]
\item[2)] In addition, we assume the Assumption \ref{sgda:smooth} holds. At t-th iteration, Assumption \ref{ass:beta-theta} holds with \[\aaa = 4\sqrt{ e} L^2\eta \exp(\frac{1}{2}\alpha^2 t\eta^2)(1+2t/n)   \text{ and }   \bb= 8\sqrt{ e} L^2\eta \exp(\frac{1}{2}\alpha^2 t\eta^2)(1+\sqrt{2t/n}). \]\end{itemize} \end{lemma}

We combine the above lemma with Lemma \ref{thm:main} to obtain bounds for SGDA with a general sampling distribution.

\begin{theorem}[{Generalization bounds for pairwise SGDA}\label{cor_sgda_smooth}]
Assume $\ell$ is $M$-bounded and Assumptions \ref{def:sgda:lip}, \ref{defn_strongly_convex_sgda} hold  w.r.t $\ell$. For the uniform distribution $\pbb$ and $\forall \delta \!\in \!(0, 1)$, with probability at least $1-\delta$ over draws of $S$, for all posterior sampling distribution $\qbb$ on $\left([n]\times[n]\right)^T$, we have the following results for SGDA with fixed step sizes.

\begin{itemize}
    \item[1)] For SGDA with $T$ iterations, we have
 \[
\ebb_{\phi\sim\qbb}\left[G(S,\phi)\right]\lesssim \left(\! \rm{KL}(\qbb\|\pbb)+\log\frac{1}{\delta}\right)\max\left\{L^2\eta (\sqrt{T}+T/n )\log^2 n,\frac{M}{\sqrt{n}}\right\}.
\]
\item[2)]In addition, if  the Assumption \ref{sgda:smooth} holds, we have
\begin{multline*}  \ebb_{\phi\sim\qbb}\big[G(S,\phi)\big]\lesssim
  \Big(\rm{KL}(\qbb\|\pbb)+\log(1/\delta)\Big)\\\max\Big\{ L^2\eta \exp(\alpha^2 t\eta^2)\left( \frac{T}{n} + 1 + \sqrt{\frac{T}{n}}
 \right)\log^2 n,\frac{M}{\sqrt{n}}\Big\}.
\end{multline*}
\end{itemize}
\end{theorem}
For part 1), if we choose $T=O(n^{2})$ and $\eta=O\left(T^{-3 /4}\right)$, this gives nonvacuous results of the order $\widetilde{O}(1/\sqrt{n})$. For part 2), if we choose $T=O(n)$ and $\eta=O(1 / \sqrt{n})$, this gives the bounds of the order $\widetilde{O}(1/\sqrt{n})$.


%


\section*{Conclusions \label{sec:conclusion}}
We derive stability-based PAC-Bayes bounds for randomized pairwise learning under general sampling, which can be applied to optimization methods, such as SGD and SGDA. We give generalization analysis for these methods that allow non-uniform sampling distributions to be updated during the training process. Future research could investigate other efficient sampling distributions, and PAC-Bayes based optimization algorithms. 



\section*{Appendix: Proof \label{app:main}}

We follow the ideas in \cite{guedj2021still} and  \cite{zhou2023toward} to prove Lemma~\ref{thm:main}. We first introduce some useful lemmas.
The following lemma shows some results on characterizing sub-Gaussian random variable and sub-exponential random variable.
For $\lambda>0$, let $\mathbb{E}[\exp(\lambda Z)]$ denote the moment-generating function (MGF) of $Z$. We denote $\mathbb{I}[\cdot]$ the indicator function.
\begin{lemma}
\label{lem:mgf-p}(\citealt{vershynin2018high}) Let $X$ be a random variable with $\mathbb{E}[X] =0$. We have the following equivalences for $X$:  
\begin{itemize}
    \item $
\|X\|_p = (\mathbb{E}|X|^p)^{1/p} \leq  \sqrt{p}$,
for all $p \geq 1$.
    \item There exists $K_1\geq 0$ such that, for all $\lambda \in \mathbb{R}$, $ \mathbb{E}[\exp(\lambda X)] \leq \exp(K_1 \lambda ^2).$
\end{itemize}
We have the following following equivalences for $X$:
\begin{itemize}
    \item  $\label{eq_exp_mgf1}
\|X\|_p = (\mathbb{E}|X|^p)^{1/p} \leq  p$,
for all $p \geq 1$.
    \item  For all $\lambda$ such that $|\lambda| \leq \frac{1}{2e} $, $\mathbb{E}[\exp(\lambda X) ]\leq \exp(2e^2\lambda^2).$
\end{itemize}
\end{lemma}

The following lemma gives a change of measure of the KL divergence.

\begin{lemma}[Lemma 4.10 in \citet{van2014probability}\label{lem:variation}] For any measurable function $g:\Phi\mapsto\rbb$ we have
  \[
  \log\ebb_{\phi\sim\pbb}[\exp(g(\phi))]=\sup_{\qbb}\left[\ebb_{\phi\sim\qbb}[g(\phi)]-\rm{KL}(\qbb\|\pbb)\right].
  \]
\end{lemma}
We denote the $L_p$-norm of a random variable $Z$ as $\|Z\|_p:=\big(\ebb[|Z|^p]\big)^{\frac{1}{p}},p\geq1$ and denote $S\backslash\{z_i\}$ the set $\{z_1,\ldots,z_{i-1},z_{i+1},\ldots,z_n\}$, and abbreviate $\sum_{i,j \in [n]:i\neq j}$ as $\sum_{i\neq j}$. For $z_k'\in \zcal$, $S^{(k)}$ is the set derived by replacing the $k$-th element of $S$ with $z_k'$.

The following lemma gives moment bounds for a summation of weakly dependent and mean-zero random functions with bounded increments under a small change.

\begin{lemma}[Theorem 1 in \citealt{lei2020sharper}\label{lem:feldman}]
  Let $S=\{z_1,\ldots,z_n\}$ be a set of independent random variables that each takes values in $\zcal$ and $M>0$. Let $g_{i,j},\forall i,j \in [n], i \neq j $ be some functions that can be decomposed as $g_{i,j}=g_j^{(i)}+\Tilde{g}_i^{(j)}$.
  Suppose for $g_j^{(i)}:\zcal^n\mapsto\rbb$ and $\Tilde{g}_i^{(j)}:\zcal^n\mapsto\rbb$, the following hold for any $i,j \in [n], i \neq j $
  \begin{itemize}
    \item $\big|\ebb_{S\backslash\{z_j\}}[g_j^{(i)}(S)]\big|\leq 2M, \quad\textbf{and}\quad \big|\ebb_{S\backslash\{z_i\}}[\Tilde{g}_i^{(j)}(S)]\big|\leq 2M$ almost surely (a.s.),
    \item $\ebb_{z_j}\big[g_j^{(i)}(S)\big]=0, \quad\textbf{and}\quad  \ebb_{z_i}\big[\Tilde{g}_i^{(j)}(S)\big]=0$ a.s.,
    \item for any $j\in[n]$ with $i\neq j, k\neq j$ we have
    $\big|g_j^{(i)}(S)-g_j^{(i)}(S^{(k)})\big|\leq 2\beta$ a.s., \quad\textbf{and}\quad  for any $i\in[n]$ with $j\neq i$ and $k\neq i$, we have
    $\big|\Tilde{g}_i^{(j)}(S)-\Tilde{g}_i^{(j)}(S^{(k)})\big|\leq 2  \beta$ a.s.
  \end{itemize}

  Then, we can decompose $\sum_{i\neq j}g_j^{(i)}(S)$ and $\sum_{i\neq j}\Tilde{g}_i^{(j)}(S)$ as follows
  \[
  \sum_{i\neq j}g_j^{(i)}(S)=X_1+X_2, \quad\textbf{and}\quad \sum_{i\neq j}^{}\Tilde{g}_i^{(j)}(S)=\Tilde{X}_1+\Tilde{X}_2
  \]
  where $X_1$, $X_2$, $\Tilde{X}_1$, $\Tilde{X}_2$ are four random variables satisfying $\ebb[X_1]=\ebb[X_2]=\ebb[\Tilde{X}_1]=\ebb[\Tilde{X}_2] =0$. Furthermore for any $p\geq1$
  \[
  \|X_1\|_p\leq 8M\sqrt{p(n-1)}n  \mkern5mu\textbf{and}\mkern5mu \|\Tilde{X}_1\|_p \leq 8M\sqrt{p(n-1)}n
  \]
  and for any $p\geq2$
  \[
  \|X_2\|_p\leq 24\sqrt{2}p(n-1)n \beta\lceil\log_2 (n-1)\rceil  \mkern5mu\textbf{and}\mkern5mu \|\Tilde{X}_2\|_p \leq 24\sqrt{2}p(n-1)n \beta\lceil\log_2 (n-1)\rceil.
  \]
\end{lemma}

\begin{proof}[Proof of Lemma \ref{thm:main}]

Based on the Lemma \ref{lem:variation}, if we set $g(\phi) = \lambda h(\phi)$, then
\begin{equation}
\ebb_{ \qbb}[h(\phi )] \leq \frac{1}{\lambda}\left( \log \ebb_\pbb[\exp(\lambda h(\phi ))]  + \rm{KL}(\qbb\|\pbb) \right).\label{eq:va}\end{equation}
To control the deviations of  $\log \ebb_\pbb[\exp(\lambda h(\phi ))]$, we use Markov's inequality. With a probability  $1-\epsilon$, we have
\[
\ebb_\pbb\left[e^{\lambda h(\phi )  }\right] \leq \frac{\ebb_S  \ebb_\pbb \left[e^{\lambda h( \phi ) }\right] }{\epsilon}.\]
Applying the above results to Eq. \eqref{eq:va}, with a probability  $1-\epsilon$, we get
\begin{equation}
  \ebb_{ \qbb}[h(\phi )]\!\leq\!\frac{1}{\lambda}( \log\! \ebb_\pbb\!\left[e^{\lambda h(\phi )  }\right] \!+\! \rm{KL}(\qbb\|\pbb))\leq \frac{1}{\lambda}\big(\log \frac{\ebb_S  \ebb_\pbb \left[e^{\lambda h( \phi ) }\right]}{\epsilon}+ \rm{KL}(\qbb\|\pbb) \big).\label{eq:markov}\end{equation}

We can exchange $\ebb_\pbb$ and  $\ebb_S$ using Fubini's theorem. Next, we will bound the generalization gap w.r.t. $\pbb$.
Let $\delta=1/n$. We denote $\Omega_\delta$ a subset with $\mbox{Pr}(\Omega_\delta)\geq1-\delta$ on which the Assumption \ref{ass:beta-theta} holds and $\Omega_{\delta}^c$ the complement of $\Omega_\delta$.
We first give results for any fixed $\phi \in \Omega_\delta$.  Given $\phi \in \Omega_\delta$, it was shown in~\citet{lei2020sharper}, $\forall i, j \in[n]$,
\[
G( S,\phi )\leq 4\beta_\phi+\frac{1}{n(n-1)}\sum_{i\neq j}g_{i,j}(S),
\]
\[
g_{i, j}(S)=\mathbb{E}_{z_{i}^{\prime}, z_{j}^{\prime}}\left[\mathbb{E}_{Z, \tilde{Z}}\left[\ell\left(A\left(S_{i, j}\right) ; Z, \tilde{Z}\right)\right]-\ell\left(A\left(S_{i, j}\right) ; z_{i}, z_{j}\right)\right].
\]
As shown in~\citet{lei2020sharper}, $g_{i, j}$ satisfies all the conditions in Lemma \ref{lem:feldman} and therefore one can apply Lemma \ref{lem:feldman} to show the existence of four random variables $X_1$, $X_2$, $\Tilde{X}_1$, $\Tilde{X}_2$ such that
$\ebb[X_1]=\ebb[X_2]=\ebb[\Tilde{X}_1]=\ebb[\Tilde{X}_2] =0$
\[
\frac{1}{n(n-1)}\sum_{i\neq j}g_{i,j}(S)=X_1+X_2+\Tilde{X}_1 +\Tilde{X}_2
\]
\begin{gather*}
\text{and}\quad \|X_1\|_p\leq 8\sqrt{ p}M(n-1)^{-\frac{1}{2}},\mkern5mu\forall p\geq1, \quad
\|\Tilde{X}_1\|_p\leq 8\sqrt{ p}M(n-1)^{-\frac{1}{2}},\mkern5mu \forall p\geq1,\\
\|X_2\|_p  \leq 24\sqrt{2}p\beta_\phi\lceil\log_2 (n-1)\rceil,\mkern4mu\forall p\geq2,\mkern4mu
\|\Tilde{X}_2\|_p \leq 24\sqrt{2}p\beta_\phi\lceil\log_2(n-1)\rceil,\mkern4mu\forall p\geq2.
\end{gather*}
By the first part of Lemma \ref{lem:mgf-p} with $X=X_1/8M(n-1)^{-\frac{1}{2}}$ to get 
\begin{equation}\label{mgf-x1}
  \max\{\ebb_S[\exp(\lambda X_1)], \ebb_S[\exp(\lambda \Tilde{X}_1)]\}\leq \exp(64M^2(n-1)^{-1}K_1\lambda^2)
\end{equation}
and by the second part of Lemma \ref{lem:mgf-p} with $X=X_2/24\sqrt{2}\beta_\phi\lceil\log_2(n-1)\rceil$,
\begin{multline}\label{mgf-x2}
\max\{\ebb_S[\exp(\lambda X_2)],\ebb_S[\exp(\lambda \Tilde{X}_2)]\}\leq \exp[2304e^2  \beta_\phi ^2\lceil\log_2(n-1)\rceil^2\lambda^2],\\
  \forall|\lambda|\leq \frac{1}{48e\sqrt{2} \beta_\phi \lceil\log_2(n-1)\rceil}.
\end{multline}
According to Jensen's inequality, we have
\begin{multline*}
\exp(\lambda X_1+\lambda X_2+\lambda \Tilde{X}_1+\lambda \Tilde{X}_2)  =\exp(\lambda X_1)\exp(\lambda X_2)\exp(\lambda \Tilde{X}_1)\exp(\lambda \Tilde{X}_2)\\
\leq \frac{1}{4}(\exp(4\lambda X_1)+\exp(4\lambda X_2) + \exp(4\lambda \Tilde{X}_1) + \exp(4\lambda \Tilde{X}_2)). 
\end{multline*}
This implies
\begin{align*}
  &\ebb_S \exp[\lambda  G( S,\phi ) ] \leq \ebb_S\exp[ \lambda ( 4\beta_\phi+X_1+ X_2 + \Tilde{X}_1+  \Tilde{X}_2 )]\notag \\
  &\leq \exp(4\lambda\beta_\phi)\frac{1}{4}\Big(\ebb_S[\exp(4\lambda X_1)+\exp(4\lambda X_2)+\exp(4\lambda \Tilde{X}_1)+\exp(4\lambda \Tilde{X}_2)]\Big). \label{general_bound}
\end{align*}
As the Assumption \ref{ass:beta-theta} $
\beta_\phi\leq \aaa+\bb\log(1/\delta)$ holds when $\phi \in \Omega_\delta$, the above inequality together with Eq. \eqref{mgf-x1}-\eqref{mgf-x2} 
imply that, for all \[0<\lambda\leq \frac{1}{192e\sqrt{2}\big(\aaa+c\log(1/\delta)\big)\lceil\log_2(n-1)\rceil},\]
we have
\begin{multline}
    \ebb_S [\exp (\lambda G( S,\phi ) )  ]\leq     \exp(4\lambda(\aaa+c\log(1/\delta)))  (\exp(256M^2(n-1)^{-1}K_1\lambda^2) \\ + \exp(9216\times(2e)^2 (\aaa+c\log(1/\delta) )^2\lceil\log_2(n-1)\rceil^2\lambda^2) ). \label{eq:gengap}
\end{multline}

Next, we give results for any fixed $\phi$.
We define $H:\zcal^n\times\Phi\mapsto\rbb$ as
$H(S,\phi)= G( S,\phi ) \mathbb{I}[\phi\in\Omega_\delta]$, where $\mathbb{I}[\cdot]$ is the indicator function. We have
\begin{equation}
\ebb_{ \qbb}[G( S,\phi )] =\ebb_{ \qbb}[H(S,\phi)]+  \ebb_{ \qbb}[G( S,\phi )|\phi\in \Omega_\delta^c]\qbb(\Omega_\delta^c).\label{pca-0}
\end{equation}
Based on Eq. (A.8) and Eq. (A.9) in~\citet{zhou2023toward}, for $\alpha >1$, we have
\begin{equation}\label{pac-0}
  \ebb_{ \qbb}[ G( S,\phi ) ] \leq \ebb_{ \qbb}[H(S,\phi)]+M\inf_{\alpha>1}\delta^{\frac{\alpha-1}{\alpha}}\Big(\ebb_{\pbb}\Big[\Big(\frac{\qbb(\phi)}{\pbb(\phi)}\Big)^\alpha\Big]\Big)^{\frac{1}{\alpha}},
\end{equation}
where $\ell(A(S;\phi))\in[0,M]$
and
\begin{align}
\ebb_{S}\ebb_{ \pbb}[\exp(\lambda H(S,\phi))] \leq \ebb_{S}\ebb_{\pbb}\big[\exp\big(\lambda\big(G( S,\phi )\big)|\phi\in \Omega_\delta\big)\big]+\delta.\label{pac-1}
\end{align}

Combining the above Eq. \eqref{pac-1} with Eq. \eqref{eq:gengap}, we obtain
\begin{multline}    \ebb_{\pbb}\ebb_S[\exp(\lambda H(S,\phi))] \leq  \exp(2\lambda(\aaa+c\log(1/\delta)))\times  \big(\exp(256M^2(n-1)^{-1}K_1\lambda^2)+\\\exp(9216\times(2e)^2\big(\aaa+\bb\log(1/\delta)\big)^2\lceil\log_2(n-1)\rceil^2\lambda^2)\big)+\delta.\label{temp__1}
\end{multline}
For any $u,v,w>0$ and $\delta\in(0,1)$, we have
\[\exp(u)(\exp(v)+\exp(w))+\delta\le \exp(u+1/2)(\exp(v)+\exp(w)).
\]
Applying the above inequality into  Eq. \eqref{temp__1}, if $u=2\lambda(\aaa+\bb\log(1/\delta))$, $v=256M^2n^{-1}K_1\lambda^2$, $w=9216\times(2e)^2 \big(\aaa+\bb\log(1/\delta)\big)^2\lceil\log_2n\rceil^2\lambda^2$, it gives
\begin{multline}
\ebb_{\pbb}\ebb_S[\exp(\lambda H(S,\phi))] \leq  \exp\big(2\lambda(\aaa+b\log(1/\delta))+1/2\big)\times \big(\exp(256M^2 \frac{K_1\lambda^2}{n-1})+ \\ \exp(9216\times(2e)^2\big(\aaa+\bb\log(\frac{1}{\delta})\big)^2\lceil\log_2(n-1)\rceil^2\lambda^2)\big).\label{pca-2}
\end{multline}
We choose
\begin{equation}\label{pca-3}
\lambda=\min\Big\{\frac{1}{192e\sqrt{2}\big(\aaa+\bb\log (1/\delta)\big)\lceil\log_2(n-1)\rceil},\frac{\sqrt{(n-1)}}{16\sqrt{K_1}M}\Big\},
\end{equation}
so that we have
\begin{gather*}
  2\lambda(\aaa+\bb\log(1/\delta)) + 1/2\leq1, \\
  256M^2(n-1)^{-1}K_1\lambda^2\leq1, \\  9216\times(2e)^2\big(\aaa+\bb\log(1/\delta)\big)^2\lceil\log_2(n-1)\rceil^2\lambda^2\leq 1.
\end{gather*}
Plugging this back into Eq. \eqref{pca-2} yields the MGF of our truncated generalization gap, $H(S;\phi)$, which is a key quantity in PAC-bays analysis
\[
\ebb_{ \pbb}\ebb_S[\exp(\lambda H(S,\phi))]\leq e(e+e)\leq e^3.
\]
Applying the above results to Eq.\eqref{eq:markov}, we have,  with a probability  $1-\delta'$,
\[  \ebb_{ \qbb}[ H( S,\phi ) ] \leq \frac{1}{\lambda}(\log (e^3/\delta') + \rm{KL}(\qbb\|\pbb)) = \frac{1}{\lambda}(3 +\log(1/\delta') + \rm{KL}(\qbb\|\pbb)).
\]
Based on the above inequality and Eq. \eqref{pac-0}, Eq. \eqref{pac-1}, the following inequality holds uniformly for all $\qbb$ with probability at least $1-\delta'$
\begin{align*}
& \ebb_{\phi\sim\qbb}[ G( S,\phi ) ] \leq \ebb_{\phi\sim\qbb}[H(S,\phi)]+M\inf_{\alpha>1}\delta^{\frac{\alpha-1}{\alpha}}\Big(\ebb_{\pbb}\Big[\Big(\frac{\qbb(\phi)}{\pbb(\phi)}\Big)^\alpha\Big]\Big)^{\frac{1}{\alpha}} \\ & \leq\frac{\rm{KL}(\qbb\|\pbb) + \log(1/\delta')  +3 }{\lambda} +M\inf_{\alpha>1}\delta^{\frac{\alpha-1}{\alpha}}\Big(\ebb_{\pbb}\Big[\Big(\frac{\qbb(\phi)}{\pbb(\phi)}\Big)^\alpha\Big]\Big)^{\frac{1}{\alpha}}.
\end{align*}
If we choose $\alpha=6$ in Lemma \ref{thm:main}, with $\delta=1/n$, we have
\begin{multline*}     \ebb_{\phi\sim\qbb}\left[G(S,\phi)\right] \lesssim  Mn^{-5/6}\left(\ebb_{\phi\sim\pbb}\left[\left(\frac{\qbb(\phi)}{\pbb(\phi)}\right)^6\right]\right)^{\frac{1}{6}}\\
  +\left(\rm{KL}(\qbb\|\pbb)+\log(1/\delta_1)\right)\max\left\{(\aaa+\bb\log(n))\lceil\log_2n\rceil,\frac{M}{\sqrt{n}}\right\}.
 \end{multline*}
In the above inequality, comparing the first term with the second term, the first term is negligible.
Therefore, our analysis shows
 \[ \ebb_{\phi\sim\qbb}\left[G(S,\phi)\right]\lesssim\left(\rm{KL}(\qbb\|\pbb)+\log(1/\delta_1)\right)\max\left\{(\aaa+\bb\log(n))\lceil\log_2n\rceil,\frac{M}{\sqrt{n}}\right\}.
 \]
 The proof is completed.\end{proof}

Here, we discuss the existence of $\ebb_{\phi\sim\pbb}\Big[\Big(\frac{\qbb(\phi)}{\pbb(\phi)}\Big)^\alpha\Big]$. In practice, we consider $\qbb$ and $\pbb$ to be sampling distributions. In these cases, $\qbb$ and $\pbb$ are discrete distributions on the same dataset. In particular, we are interested in the case with $\pbb$ being the uniform distribution. Under these circumstances, this expectation exists.

\section*{Proofs on Applications }\label{app:app}


\subsection*{Stochastic Gradient Descent \label{prf:sgd}}

We will prove that stability bounds of SGD meet the Assumption \ref{ass:beta-theta}. Based on this, we can derive the generalization bounds for SGD with smooth and non-smooth convex loss functions. To this aim, we introduce the following lemma to bound the summation of i.i.d events \citep{shalev2014understanding}.

\begin{lemma}[Chernoff's Bound\label{Chernoff's Bound}]
Let $Z_1, \ldots , Z_t$ be independent random variables taking values in $\{0, 1\}$. Let $Z = \sum_{k=1}^{t} Z_k$ and $\mu = \mathbb{E}[Z]$.  Then for any $\delta \in (0,1)$ with probability at least $1- \delta$ we have
\[
Z \leq \mu + \log(1/\delta) + \sqrt{2\mu \log(1/\delta)}.
\]
\end{lemma}

We first present the stability bounds for non-smooth and convex cases.

\begin{proof}[Proof of Lemma \ref{sta_sgd}, 1)]
Without loss of generality, we assume $S$ and $S'$ differ by the last example. Based on the Eq. (F.2) in~\citet{lei2021generalizationb}, we have 
\begin{align*}\|\bw_{t+1} - \bw_{t+1}'\|^2_2 &\leq 4L^2\eta^2 (1+p)^{\sum_{k=1}^{t} \ibb[i_k=n\mkern5mu\text{ or }\mkern5mu j_k =n] }\Big(  t+ p^{-1} \sum_{k=1}^{t} \ibb[i_k=n\mkern5mu\text{ or }\mkern5mu j_k =n] \Big).
\end{align*}
We set $p=1/\sum_{k=1}^{t} \ibb[i_k=n\mkern5mu\text{ or }\mkern5mu j_k =n]$ and use the inequality $(1+1/x)^x\leq e$ to get
\[
\|\bw_{t+1} - \bw_{t+1}'\|^2_2
\leq 4eL^2\eta^2\Big(  t+ \Big(\sum_{k=1}^{t} \ibb[i_k=n\mkern5mu\text{ or }\mkern5mu j_k =n]\Big)^2 \Big).
\]
It then follows that
\[
\|\bw_{t+1} - \bw_{t+1}'\|_2
\leq 2\sqrt{e}L\eta\Big(  \sqrt{t}+ \sum_{k=1}^{t} \ibb[i_k=n\mkern5mu\text{ or }\mkern5mu j_k =n]\Big).
\]
According to the Lipschitz continuity, we know that SGD is $\beta_\phi$-uniformly stable with
\begin{equation} \label{eqa:nonsmooth_bound}
\beta_\phi = 2\sqrt{e}L^2\eta\Big(  \sqrt{t}+ \max_{k\in[n]}\sum_{m=1}^{t}\ibb[i_m=k\mkern5mu\text{ or }\mkern5mu j_m =k]\Big).
\end{equation}
To bound $\beta_\phi$ with high probability, we set $\beta_{\phi,k}=2\sqrt{e}L^2\eta\big(\sqrt{t}+ \sum_{m=1}^{t} \ibb[i_m=k\mkern5mu\text{ or }\mkern5mu j_m =k]\big)$, and note that
$\mathbb{E} [ \mathbb{I}[i_m=k\mkern5mu\text{ or }\mkern5mu j_m =k]] \leq \mbox{Pr}\{i_m=k\} +\mbox{Pr}\{j_m=k\} = 2/n.$
Applying Lemma \ref{Chernoff's Bound} to the sum in Eq. \eqref{eqa:nonsmooth_bound}, with probability at least $1- \delta/n$, we get
\begin{align*}
\beta_{\phi,k} \leq 2\sqrt{e}L^2\eta (\sqrt{t}+ 2t/n + \log(n/\delta) + 2\sqrt{ t/n\log(n/\delta)}).
\end{align*}
Therefore, with probability at least $1-\delta$, the following holds simultaneously for all $k\in[n]$ by the union bound on probability
\[
\beta_{\phi,k} \leq 2\sqrt{e}L^2\eta (\sqrt{t}+ 2t/n + \log(n/\delta) + 2\sqrt{ t/n\log(n/\delta)}).
\]
For $\delta\in(0,1/n)$, this implies the following inequality with probability at least $1-\delta$
\begin{align}
\beta_\phi\leq 2\sqrt{e}L^2\eta (\sqrt{t}+ 2t/n + 2\log(1/\delta) + 2\sqrt{ 2 t/n \log(1/\delta)}). \label{beta_theta}
\end{align}
Finally, from Eq. \eqref{beta_theta} we know that SGD with the uniformly distributed hyperparameter $\phi$ meets Assumption \ref{ass:beta-theta} with
\[
\aaa = 2\sqrt{e}L^2\eta (\sqrt{t}+ 2t/n), \quad\bb=4\sqrt{e}L^2\eta (1+ \sqrt{2t/n}).
\]
The proof is completed.
\end{proof}

\begin{proof}[Proof of Lemma \ref{sta_sgd}, 2)]
By an intermediate result in the proof in Lemma C.3 of \citet{lei2020sharper}, for all $z, \tilde{z} \in \mathcal{Z}$ and $i_k, j_k\in [n], i_k\neq j_k$, with $L$-Lipschitz, we have
\[
\left|\ell\left(\mathbf{w}_{t+1} ; z, \tilde{z} \right)-\ell\left(\mathbf{w}_{t+1}; z, \tilde{z}\right)\right| \leq L \|\mathbf{w}_{t+1}-\mathbf{w}_{t+1}^{\prime} \|_{2}  \leq 2 L^{2} \sum_{k=1}^{t} \eta_{k} \mathbb{I}\left[i_{k}=n \text { or } j_{k}=n\right] .\]
From this inequality it follows that SGD is $\beta_\phi$-uniformly stable with
\begin{align}\label{convex:bound}
\beta_\phi = 2L^2 \max_{k\in[n] }\sum_{m=1}^{t}\eta_m  \mathbb{I}[i_m=k\mkern5mu\text{ or }\mkern5mu j_m =k].
\end{align}
Let $\beta_{\phi,k} = 2L^2 \sum_{m=1}^{t}\eta_j  \mathbb{I}[i_m=k\mkern5mu\text{ or }\mkern5mu j_m =k]$ for any $k\in[n]$.
It remains to show that the stability parameter of SGD meets  Assumption \ref{ass:beta-theta}.
Using Lemma \ref{Chernoff's Bound} with $Z_m = \mathbb{I}[i_m=k\mkern5mu\text{ or }\mkern5mu j_m =k]$ and noting that $\mathbb{E} [ \mathbb{I}[i_m=k\mkern5mu\text{ or }\mkern5mu j_m =k]] \leq  2/n$, we
get the following inequality with probability at least $1- \delta/n$ (taking $\eta_j=\eta$),
\begin{align}\label{convex:pbound}
\beta_{\phi,k}
&\leq 2L^2\eta (2t/n + \log(n/\delta) + 2\sqrt{  t/n\log(n/\delta)}).
\end{align}
By the union bound,  with probability at least $1-\delta$, Eq. \eqref{convex:pbound} holds for all $k\in[n]$. Therefore, with probability at least $1-\delta$, it gives
\begin{align*}
\beta_\phi & \leq 2L^2\eta (2t/n + \log(n/\delta) + 2\sqrt{ t/n\log(n/\delta)}) \leq 2L^2\eta (2t/n +  2\log(1/\delta) +\\
& 2\sqrt{2 t/n\log(1/\delta)}) \leq 4L^2\eta t/n + 4L^2\eta(1+ \sqrt{2t/n}) \log(1/\delta),
\end{align*}
where we have used $\delta\in(0,1/n)$ in the second inequality. Assumption \ref{ass:beta-theta} holds with 
\[\aaa = 4L^2\eta t/n, \bb=4L^2\eta(1+ \sqrt{2t/n}).\]
This completes the proof.
\end{proof}

\begin{proof}[Proof of Theorem \ref{cor_sgd_smooth} ]%
With $A(S;\phi)= \bw_T$, 
it follows from Lemma \ref{sta_sgd}, 1) and 2) that SGD with convex non-smooth and convex smooth loss functions satisfy Assumption \ref{ass:beta-theta} respectively. Applying the upper bound on $\beta_\phi$ to Lemma \ref{thm:main}, the result follows.\end{proof}

\subsection*{Stochastic Gradient Descent Ascent \label{prf:sgda}}
Next, we prove the generalization bounds for SGDA with smooth and non-smooth convex-concave loss functions.

\begin{lemma}[Lemma C.1., \citep{lei2021stability}  \label{lem:sgda}]
Let  $\ell$  be  convex-concave.
\begin{itemize}
    \item[1)] If Assumption \ref{def:sgda:lip} holds, then
\begin{multline*}
    \left\|\!\binom{\bw-\eta \nabla_{\bw} \ell(\bw, \bv)}{\bv+\eta \nabla_{\bv} \ell(\bw, \bv)}\!-\!\binom{\bw^{\prime}-\eta \nabla_{\bw} \ell (\bw^{\prime}, \bv^{\prime} )}{\bv^{\prime}+\eta \nabla_{\bv} \ell (\bw^{\prime}, \bv^{\prime} )}\!\right\|_{2}^{2} \!\leq\! \left\|\!\binom{\bw-\bw^{\prime}}{\bv-\bv^{\prime}}\!\right\|_{2}^{2}+8 L^{2} \eta^{2}.
\end{multline*}
\item[2)] If Assumption \ref{sgda:smooth} holds, then  
\begin{multline*}
    \!\left\|\!\binom{\!\bw-\eta \nabla_{\bw} \ell(\bw, \bv)\!}{\!\bv+\eta \nabla_{\bv} \ell(\bw, \bv)\!}\!-\!\binom{\!\bw^{\prime}-\eta \nabla_{\bw} \ell\left(\!\bw^{\prime}, \bv^{\prime}\!\right)\!}{\!\bv^{\prime}+\eta \nabla_{\bv} \ell\left(\!\bw^{\prime}, \bv^{\prime}\!\right)\!}\!\right\|_{2}^{2} \!\leq\!\left(\!1 \!+\!\alpha^{2} \eta^{2}\!\right)\!\left\|\!\binom{\bw-\bw^{\prime}}{\bv-\bv^{\prime}}\!\right\|_{2}^{2}\!.
\end{multline*}
\end{itemize} \end{lemma}

\begin{proof}[Proof of Lemma \ref{lem:sta_sgda}, 1)]
We assume $S$ and $S'$ differ by the last example for simplicity.
Based on the Lemma \ref{lem:sgda} 1), 
for $i_{t}\neq n, j_{t}\neq n,  \text { and }  i_{t}\neq j_{t}$, we have
\begin{equation}
\left\|\left(\arraycolsep=1.4pt\def\arraystretch{0.5}\begin{array}{c}
\mathbf{w}_{t+1}-\mathbf{w}_{t+1}^{\prime} \\
\mathbf{v}_{t+1}-\mathbf{v}_{t+1}^{\prime}
\end{array}\right)\right\|_{2}^{2} \leq
 \left\|\left(\arraycolsep=1.4pt\def\arraystretch{0.5}\begin{array}{c}
\bw_{t}-\bw_{t}^{\prime} \\
\bv_{t}-\bv_{t}^{\prime}
\end{array}\right)\right\|_{2}^{2} + 8L^2\eta_t^2. \label{eq:sgda_ns}\end{equation}

When $i_{t}= n \text { or } j_{t}= n,i_{t}\neq j_{t}$, we have
\begin{align}
&\Big\|\!\Big(\!
\arraycolsep=1.4pt\def\arraystretch{0.5}\begin{array}{c}
\bw_{t+1}-\bw_{t+1}' \\
\bv_{t+1}-\bv_{t+1}'
\end{array}\!\Big)\!\Big\|_{2}^{2}  \!\leq\!\Big\|\!\Big(\!\arraycolsep=1.4pt\def\arraystretch{0.5}\begin{array}{c}
\bw_{t}-\eta_{t} \nabla_{\bw} \ell\bw_{t}, \bv_{t} ; z_{i_t},{z}_{j_t} )-\bw_{t}'+\eta_{t} \nabla_{\bw} \ell\bw_{t}', \bv_{t}' ; z'_{i_t},{z}_{j_t}' ) \\
\bv_{t}+\eta_{t} \nabla_{\bv} \ell(\bw_{t}, \bv_{t} ; z_{i_t},{z}_{j_t} )-\bv_{t}'-\eta_{t} \nabla_{\bv} \ell(\bw_{t}', \bv_{t}' ; z_{i_t}',{z}_{j_t}' )
\end{array}\!\Big)\!\Big\|_{2}^{2}\notag \\
& \!\leq\!(1+p)\!\Big\|\!\Big(\!\arraycolsep=1.4pt\def\arraystretch{0.4}\begin{array}{c}
\bw_{t}-\bw_{t}' \\
\bv_{t}-\bv_{t}'
\end{array}\!\Big)\!\Big\|_{2}^{2}+(1+ \frac{1}{p}) \eta_{t}^{2}\Big\|\Big(\arraycolsep=1.1pt\def\arraystretch{0.5}\begin{array}{c}
\nabla_{\bw} \ell(\bw_{t}, \bv_{t} ; z_{i_t},{z}_{j_t} )-\nabla_{\bw} \ell(\bw_{t}', \bv_{t}' ; z_{i_t}',{z}_{j_t}' ) \\
\nabla_{\bv} \ell(\bw_{t}, \bv_{t} ; z_{i_t},{z}_{j_t})-\nabla_{\bv} \ell(\bw_{t}', \bv_{t}' ; z_{i_t}',{z}_{j_t}' )
\end{array}\big)\Big\|_{2}^{2}\notag \\ &\leq(1+p)\Big\|\Big(\arraycolsep=1.4pt\def\arraystretch{0.5}\begin{array}{c}
\bw_{t}-\bw_{t}' \\
\bv_{t}-\bv_{t}'
\end{array}\Big)\Big\|_{2}^{2}+8(1+1 / p) \eta_{t}^{2}L^2,\label{eq:sgda:gen1}
\end{align}
where in the second inequality, we use that, for any $p>0$, we have $(c+d)^{2} \leq(1+p) c^{2}+(1+1 / p) d^{2}$. Combining Eq. \eqref{eq:sgda_ns} and Eq. \eqref{eq:sgda:gen1}, this gives
\begin{multline*}
\left\|\left(\arraycolsep=1.4pt\def\arraystretch{0.5}\begin{array}{c}
\bw_{t+1}-\bw_{t+1}' \\
\bv_{t+1}-\bv_{t+1}'
\end{array}\right)\right\|_{2}^{2} \leq\left( \left\|\left(\arraycolsep=1.4pt\def\arraystretch{0.5}\begin{array}{c}
\bw_{t}-\bw_{t}' \\
\bv_{t}-\bv_{t}'
\end{array}\right)\right\|_{2}^{2} + 8L^2\eta_t^2\right)\ibb[i_t\neq n\mkern5mu \text{and} \mkern5mu j_t \neq n]+ \\ \left((1+p)\left\|\left(\arraycolsep=1.4pt\def\arraystretch{0.5}\begin{array}{c}
\bw_{t}-\bw_{t}' \\
\bv_{t}-\bv_{t}'
\end{array}\right)\right\|_{2}^{2}+8(1+1 / p) \eta_{t}^{2} L^{2}\right) \ibb[i_t= n\mkern5mu \text{ or } \mkern5mu j_t = n]\\
\leq \left(1+p  \ibb[i_t= n\mkern5mu \text{ or } \mkern5mu j_t = n]\right)\left\|\left(\arraycolsep=1.4pt\def\arraystretch{0.5}\begin{array}{c}
\mathbf{w}_{t}-\mathbf{w}_{t}^{\prime} \\
\mathbf{v}_{t}-\mathbf{v}_{t}^{\prime}
\end{array}\right)\right\|_{2}^{2}+8 L^{2} \eta_{t}^{2}\left(1+ \ibb[i_t= n\mkern5mu \text{ or } \mkern5mu j_t = n] / p\right).\end{multline*}

We apply the above inequality recursively and follow the analysis of Eq. (C.4) in~\citet{lei2021stability}:
\begin{align*}\label{sgda:smooth}
&\left\|\left(\arraycolsep=1.4pt\def\arraystretch{0.5}\begin{array}{c}
\mathbf{w}_{t+1}-\mathbf{w}_{t+1}^{\prime} \\
\mathbf{v}_{t+1}-\mathbf{v}_{t+1}^{\prime}
\end{array}\right)\right\|_{2}^{2} \\ \leq& 8 L^{2} \eta^{2} \sum_{k=1}^{t}\left(1+\ibb[i_k= n\mkern5mu \text{ or } \mkern5mu j_k = n] / p\right) \prod_{r=k+1}^{t}\left(1+p \ibb[i_r= n\mkern5mu \text{ or } \mkern5mu j_r = n]\right) \\
= &8 L^{2} \eta^{2} \sum_{k=1}^{t}\left(1+\ibb[i_k= n\mkern5mu \text{ or } \mkern5mu j_k = n] / p\right) \prod_{r=k+1}^{t}(1+p)^{\ibb[i_r= n\mkern5mu \text{ or } \mkern5mu j_r = n]}\\
\leq & 8 L^{2} \eta^{2}(1+p)^{\sum_{k=1}^{t} \ibb\left[i_{k}=n \text { or } j_{k}=n\right]}\left(t+\sum_{k=1}^{t} \ibb\left[i_{k}=n \text { or } j_{k}=n\right] / p\right),\end{align*} where we assume the fixed step sizes.  We set $p=1/\sum_{k=1}^{t}\ibb\left[i_{k}=n \text { or } j_{k}=n\right]$ and use the inequality $(1+1/x)^x\leq e$ to derive
\[
\left\|\left(\arraycolsep=1.4pt\def\arraystretch{0.5}\begin{array}{c}
\mathbf{w}_{t+1}-\mathbf{w}_{t+1}^{\prime} \\
\mathbf{v}_{t+1}-\mathbf{v}_{t+1}^{\prime}
\end{array}\right)\right\|_{2}^{2}
 \leq 8e L^{2} \eta^{2} \left(  t+ \Big(\sum_{k=1}^{t} \ibb\left[i_{k}=n \text { or } j_{k}=n\right]\Big)^2  \right).
\]
It then follows that
\[
\left\|\left(\arraycolsep=1.4pt\def\arraystretch{0.5}\begin{array}{c}
\mathbf{w}_{t+1}-\mathbf{w}_{t+1}^{\prime} \\
\mathbf{v}_{t+1}-\mathbf{v}_{t+1}^{\prime}
\end{array}\right)\right\|_{2}
 \leq \sqrt{8e} L \eta \left(  \sqrt{t}+ \sum_{k=1}^{t} \ibb\left[i_{k}=n \text { or } j_{k}=n\right] \right).
\]

By $L$-Lipschitzness, we have
\begin{multline*}
|\ell\left(A_{\mathbf{w}}(S;\phi), A_{\mathbf{v}}(S;\phi) , z, \tilde{z} \right)-\ell\left(A_{\mathbf{w}}\left(S^{\prime};\phi\right), A_{\mathbf{v}}\left(S^{\prime};\phi\right) , z, \tilde{z}\right)| \\
\leq 2\sqrt{2e}L^2\eta\Big(  \sqrt{t}+ \max_{k\in[n]}\sum_{r=1}^{t} \ibb\left[i_{r}=k \text { or } j_{r}=k\right]\Big).
\end{multline*}
Therefore, we know that SGDA is $\beta_\phi$-uniformly stable with 
\begin{equation} \label{sgda:eqa:nonsmooth_bound}
\beta_\phi =  2\sqrt{2e}L^2\eta\Big(  \sqrt{t}+ \max_{k\in[n]}\sum_{r=1}^{t} \ibb\left[i_{r}=k \text { or } j_{r}=k\right]\Big).
\end{equation}
For simplicity, let $\beta_{\phi,k}= 2\sqrt{2e}L^2\eta\big(\sqrt{t}+ \sum_{r=1}^{t} \ibb\left[i_{r}=n \text { or } j_{r}=n\right]\big)$.
Applying Lemma~\ref{Chernoff's Bound} to Eq. \eqref{sgda:eqa:nonsmooth_bound}, with probability at least $1- \delta/n$, we have 
\begin{align*}
\beta_{\phi,k} \leq2\sqrt{2e}L^2\eta (\sqrt{t}+ 2t/n + \log(n/\delta) + 2\sqrt{ t/n\log(n/\delta)}).
\end{align*}
With probability at least $1-\delta$, the following holds for all $k\in[n]$
\[
\beta_{\phi,k} \leq 2\sqrt{2e}L^2\eta (\sqrt{t}+ 2t/n + \log(n/\delta) + 2\sqrt{ t/n\log(n/\delta)}).
\]
This suggests the following inequality with probability at least $1-\delta$
\[
\beta_\phi\leq2\sqrt{2e}L^2\eta (\sqrt{t}+ 2t/n + 2\log(1/\delta) + 2\sqrt{2t/n\log(1/\delta)}).
\]
This suggests that SGDA with uniform distribution and the hyperparameter $\phi$ meets Assumption \ref{ass:beta-theta} with
\[
\aaa = 2\sqrt{ e}L^2\eta (\sqrt{t}+ 2t/n), \quad\bb= 4\sqrt{2e}L^2\eta (1+\sqrt{ 2t/n}).
\]
The proof is completed.
\end{proof}

\begin{proof}[Proof of Lemma \ref{lem:sta_sgda}, 2)]
Without loss of generality, we first assume $S$ and $S'$ differ by the last example.
Based on Lemma \ref{lem:sgda} 2), if $i_{t}\neq n \text { and } j_{t}\neq n$, we have
\begin{multline*}  \Big\|\Big(\arraycolsep=1.4pt\def\arraystretch{0.5}\begin{array}{c}
\mathbf{w}_{t+1}-\mathbf{w}_{t+1}^{\prime} \\
\mathbf{v}_{t+1}-\mathbf{v}_{t+1}^{\prime}
\end{array}\Big)\Big\|_{2}^{2} =\Big\|\Big(\arraycolsep=1.4pt\def\arraystretch{0.5}\begin{array}{c}
\bw_{t}-\eta_t \nabla_{\bw} \ell(\bw_{t}, \bv_{t};z_{i_t},{z}_{j_t}) \\
\bv_{t}+\eta_t \nabla_{\bv}\ell(\bw_{t}, \bv_{t};z_{i_t},{z}_{j_t})
\end{array}\Big)- \\   \Big(\arraycolsep=1.4pt\def\arraystretch{0.5}\begin{array}{c}
\bw_{t}^{\prime}-\eta_t \nabla_{\bw} \ell\left(\bw_{t}^{\prime}, \bv_{t}^{\prime};z_{i_t},{z}_{j_t}\right) \\
\bv_{t}^{\prime}+\eta_t \nabla_{\bv} \ell\left(\bw_{t}^{\prime}, \bv_{t}^{\prime};z_{i_t},{z}_{j_t}\right)
\end{array}\Big)\Big\|_{2}^{2}  \leq
(1 + \alpha^2 \eta_t^2)\Big\|\Big(\arraycolsep=1.4pt\def\arraystretch{0.5}\begin{array}{c}
\bw_{t}-\bw_{t}^{\prime} \\
\bv_{t}-\bv_{t}^{\prime}
\end{array}\Big)\Big\|_{2}^{2}.\end{multline*}
 When $i_{t} = n \text { or } j_{t} =  n$, we consider Eq. \eqref{eq:sgda:gen1}. Combining these two cases, we get
\begin{multline}
\Big\|\Big(\arraycolsep=1.4pt\def\arraystretch{0.5}\begin{array}{c}
\bw_{t+1}-\bw_{t+1}' \\
\bv_{t+1}-\bv_{t+1}'
\end{array}\Big)\Big\|_{2}^{2} \leq\Big(1+\alpha^{2} \eta_{t}^{2}\Big)\Big\|\Big(\arraycolsep=1.4pt\def\arraystretch{0.5}\begin{array}{c}
\bw_{t}-\bw_{t}' \\
\bv_{t}-\bv_{t}'
\end{array}\Big)\Big\|_{2}^{2} \ibb[i_t\neq n\mkern5mu \text{and} \mkern5mu j_t \neq n]+ \\ \Big((1+p)\Big\|\Big(\arraycolsep=1.4pt\def\arraystretch{0.5}\begin{array}{c}
\bw_{t}-\bw_{t}' \\
\bv_{t}-\bv_{t}'
\end{array}\Big)\Big\|_{2}^{2}+8(1+\frac{1}{p}) \eta_{t}^{2} L^{2}\Big) \ibb[i_t= n\mkern1mu \text{ or } \mkern1mu j_t = n] \\  \leq  \big(1+\alpha^{2} \eta_{t}^{2} p\ibb[i_t= n\mkern1mu \text{ or } \mkern1mu j_t = n]\big)\Big\|\Big(\arraycolsep=1.4pt\def\arraystretch{0.5}\begin{array}{c}
\bw_{t}-\bw_{t}' \\
\bv_{t}-\bv_{t}'
\end{array}\Big)\Big\|_{2}^{2}  + 8(1+\frac{1}{p}) \eta_{t}^{2} L^{2}\ibb[i_t= n\mkern1mu \text{ or } \mkern1mu j_t = n]. \label{eq:sgda_1} \end{multline}

We apply the above Eq. \eqref{eq:sgda_1} recursively, following the proof of Theorem 2(d) in~\citet{lei2021stability}, 
\begin{align*}
& \Bigg\|\Big(\arraycolsep=1.4pt\def\arraystretch{0.5}\begin{array}{c}
\bw_{t+1}-\bw_{t+1}' \\
\bv_{t+1}-\bv_{t+1}'
\end{array}\Big)\Bigg\|_{2}^{2} \\
 \leq & 8(1+1 / p) L^{2} \sum_{k=1}^{t} \eta_{k}^{2} \ibb[i_k= n\mkern5mu \text{ or } \mkern5mu j_k = n]\prod_{r=k+1}^{t}\Big(1+\alpha^{2} \eta_{r}^{2}+p \ibb[i_r= n\mkern5mu \text{ or } \mkern5mu j_r = n]\Big) \\
 \leq &  8(1+ \frac{1}{p}) L^{2} \eta^{2} \sum_{k=1}^{t} \ibb[i_k= n\mkern1mu \text{ or } \mkern1mu j_k = n]\prod_{r=k+1}^{t} (1+\alpha^{2} \eta_{r}^{2} ) \prod_{r=k+1}^{t} (1+p \ibb[i_r= n\mkern1mu \text{ or } \mkern1mu j_r = n] ) \\
= &8(1+1 / p) L^{2} \eta^{2} \sum_{k=1}^{t} \ibb[i_k= n\mkern5mu \text{ or } \mkern5mu j_k = n] \prod_{r=k+1}^{t}\Big(1+\alpha^{2} \eta_{r}^{2}\Big) \prod_{r=k+1}^{t}(1+p)^{\ibb[i_r= n\mkern5mu \text{ or } \mkern5mu j_r = n]}
\\
\leq  &  8(1+1 / p) L^{2} \eta^{2} \prod_{k=1}^{t}\Big(1+\alpha^{2} \eta_{k}^{2}\Big) \prod_{k=1}^{t}(1+p)^{\ibb\left[i_{k}=n \text { or } j_{k}=n\right]} \sum_{k=1}^{t} \ibb\left[i_{k}=n \text { or } j_{k}=n\right] \\
\leq & 8(1+1 / p) L^{2} \eta^{2} \exp\Big(\alpha^2 \sum_{k=1}^{t} \eta_{k}^{2}\Big)  (1+p)^{ \sum_{k=1}^{t} \ibb\left[i_{k}=n \text { or } j_{k}=n\right]} \sum_{k=1}^{t} \ibb\left[i_{k}=n \text { or } j_{k}=n\right],
\end{align*}
where we assume fixed step sizes and use $ 1+ x \leq e^x$ in the last inequality. We set $p=1/\sum_{k=1}^{t} \ibb[i_k=n\mkern5mu\text{ or }\mkern5mu j_k =n]$ and use the inequality $(1+1/x)^x\leq e$ to derive 
\[ \left\|\left(\arraycolsep=1.4pt\def\arraystretch{0.5}\begin{array}{c}
\mathbf{w}_{t+1}-\mathbf{w}_{t+1}^{\prime} \\
\mathbf{v}_{t+1}-\mathbf{v}_{t+1}^{\prime}
\end{array}\right)\right\|_{2}^{2}
\leq 8e\left(1 +  \sum_{k=1}^{t} \ibb\left[i_{k}=n \text { or } j_{k}=n\right]  \right)^2 L^{2} \eta^{2} \exp\left( \alpha^2 \sum_{k=1}^{t} \eta_{k}^{2}\right)  .
\]
Based on the $L$-Lipschitzness and the above inequality, for any two neighboring datasets $ S, S'\in \mathcal{Z}^n,\forall z , \tilde{z} \in \mathcal{Z},$ we have 
\begin{multline*}
|\ell\left(A_{\mathbf{w}}(S;\phi), A_{\mathbf{v}}(S;\phi) , z, \tilde{z} \right)-\ell\left(A_{\mathbf{w}}\left(S^{\prime};\phi\right), A_{\mathbf{v}}\left(S^{\prime};\phi\right) , z, \tilde{z}\right)| \\
\leq 4\sqrt{e} L^2\eta \exp(\frac{1}{2}\alpha^2 t\eta^2) \max_{k\in[n]} \left(1 +  \sum_{r=1}^{t} \ibb\left[i_{r}=k \text { or } j_{r}=k\right]\right).
\end{multline*}
Therefore, we know that SGDA is $\beta_\phi$-uniformly stable with
\begin{align*}
\beta_\phi = 4\sqrt{ e} L^2\eta \exp(\frac{1}{2}\alpha^2 t\eta^2) \max_{k\in[n]} \left(1 +  \sum_{r=1}^{t}  \ibb\left[i_{r}=k \text { or } j_{r}=k\right]\right).
\end{align*}
For simplicity, let
$\beta_{\phi,k} = 4\sqrt{ e} L^2\eta \exp(\frac{1}{2}\alpha^2 t\eta^2) \left( 1+\sum_{r=1}^{t}   \ibb\left[i_{r}=k \text { or } j_{r}=k\right] \right)$ for any $k\in[n]$.
Taking the expectation over both sides of above inequality, we derive
\begin{equation}\label{stab-sgda-1}
\aaa = 4\sqrt{ e} L^2\eta \exp(\frac{1}{2}\alpha^2 t\eta^2)(1+2t/n),
\end{equation}
where $\mathbb{E} [\ibb\left[i_{r}=k \text { or } j_{r}=k\right]] \leq 2/n.$
Applying $Z_r = \ibb\left[i_{r}=k \text { or } j_{r}=k\right]$ in Lemma~\ref{Chernoff's Bound}, we
get the following inequality with probability at least $1- \delta/n$
\begin{align}\label{convex:pbound_sgda}
\beta_{\phi,k}
&\leq 4\sqrt{ e} L^2\eta \exp(\frac{1}{2}\alpha^2 t\eta^2)(1+2t/n + \log(n/\delta) + 2\sqrt{ t/n\log(n/\delta)}).
\end{align}
By the union bound in probability,  with probability at least $1-\delta$, Eq. \eqref{convex:pbound_sgda} holds for all $k\in[n]$. Therefore, with probability at least $1-\delta$
\begin{align*}
\beta_\phi & \leq 4\sqrt{ e} L^2\eta \exp(\frac{1}{2}\alpha^2 t\eta^2)(1+2t/n + \log(n/\delta) + 2\sqrt{ t/n\log(n/\delta)})\\
&\leq 4\sqrt{ e} L^2\eta \exp(\frac{1}{2}\alpha^2 t\eta^2)(1+2t/n + 2\log(1/\delta) + 2\sqrt{ 2t/n\log(1/\delta)})\\
& \leq 4\sqrt{ e} L^2\eta \exp(\frac{1}{2}\alpha^2 t\eta^2)(1+2t/n) + 8\sqrt{ e} L^2\eta \exp(\frac{1}{2}\alpha^2 t\eta^2)(1+\sqrt{2t/n} )\log(1/\delta) \\
& \leq \aaa + 8\sqrt{ e} L^2\eta \exp(\frac{1}{2}\alpha^2 t\eta^2)(1+\sqrt{2t/n}) \log(1/\delta) ,
\end{align*}
where we have used $\delta\in(0,1/n)$ in the second inequality, and Eq. \eqref{stab-sgda-1} in the last inequality.
Therefore, Assumption \ref{ass:beta-theta} holds with $\aaa = 4\sqrt{ e} L^2\eta \exp(\frac{1}{2}\alpha^2 t\eta^2)(1+2t/n)$ and $\bb= 8\sqrt{ e} L^2\eta \exp(\frac{1}{2}\alpha^2 t\eta^2)(1+\sqrt{2t/n})$. The proof is completed. \end{proof}

Based on the above lemma, we are ready to develop generalization bounds in Theorem \ref{cor_sgda_smooth} for SGDA with smooth and non-smooth loss functions.

\begin{proof}[Proof of Theorem \ref{cor_sgda_smooth}]%


With  $A(S;\phi) =\left(A_{\mathrm{w}}(S;\phi), A_{\mathrm{v}}(S;\phi)\right)$, 
based on Lemma \ref{lem:sta_sgda}, 1) and 2), SGDA with convex-concave non-smooth and convex-concave smooth loss functions satisfy Assumption \ref{ass:beta-theta}  respectively. Applying the upper bounds on $\beta_\phi$ to Lemma \ref{thm:main}, we derive the result.\end{proof}

\setlength{\bibsep}{0.00001cm}
\small

\bibliographystyle{apalike}

\end{document}